\documentclass{article}
\usepackage{ijcai21}

\usepackage{times}
\usepackage{soul}
\usepackage{url}
\usepackage[hidelinks]{hyperref}

\usepackage[small]{caption}
\usepackage{graphicx}
\usepackage{amsmath}
\usepackage{amsthm}
\usepackage{algorithm}
\graphicspath{{figures/}}
\usepackage{latexsym}
\usepackage{soul}
\usepackage{amsmath}
\usepackage{thmtools}
\usepackage{thm-restate}

\usepackage[utf8]{inputenc} 
\usepackage[T1]{fontenc}    
\usepackage{url}            
\usepackage{booktabs}       
\usepackage{amsfonts}       
\usepackage{nicefrac}       
\usepackage{microtype}      
\usepackage{color}

\title{A Vertical Federated Learning Framework for Graph Convolutional Network}

\usepackage{amsthm}
\newtheorem{theorem}{Theorem}[section]

\author{
Xiang Ni$^1$\and
Xiaolong Xu$^{1}$
\and
Linjuan Lyu$^1$\and
Changhua Meng$^1$\and
Weiqiang Wang$^1$
\\
\affiliations
$^1$Ant Group\\
\emails
nixiang85@gmail.com, lingjuanlvsmile@gmail.com\{yiyin.xxl,changhua.mch,weiqiang.wwq\}@antgroup.com
}

\begin{document}

\maketitle
\begin{abstract}
Recently, Graph Neural Network (GNN) has achieved remarkable success in various real-world problems on graph data. 
However in most industries, data exists in the form of isolated islands and the data privacy and security is also an important issue. 
In this paper, we propose FedVGCN, a federated GCN learning paradigm for privacy-preserving node classification task under data vertically partitioned setting, which can be generalized to existing GCN models. Specifically, we split the computation graph data into two parts. 
For each iteration of the training process, the two parties transfer intermediate results to each other under homomorphic encryption. 
We conduct experiments on benchmark data and the results demonstrate the effectiveness of FedVGCN in the case of GraphSage.
\end{abstract}

\section{Introduction}
The protection of user privacy is an important concern in machine learning, as evidenced by the rolling out of the General Data Protection Regulation (GDPR) in the European Union (EU) in May 2018 (\cite{chi2018privacy}). The GDPR is designed to give users to protect their personal data, which motivates us to explore machine learning frameworks with data sharing while not violating user privacy (\cite{balle2018improving,bonawitz2019towards}).

Privacy is an important challenge when data aggregation and collaborative learning happens across different entities~\cite{lyu2020threats,lyu2020privacy}. To address this privacy issue many endeavors have been taken in different directions, among which, two important techniques are differential privacy (\cite{acar2018survey,aono2016scalable}) and fully homomorphic encryption (\cite{zhang2015privacy}). Recent advance in fully homomorphic encryption (FHE) (\cite{yuan2013privacy}) allows users to encrypt data with the public key and offload computation to the cloud. The cloud computes on the encrypted data and generates encrypted results. Without the secret key, cloud simply serves as a computation platform but cannot access any user information. This powerful technique has been integrated with deep learning in the pioneering work of convolutional neural network (\cite{lin2013network,veidinger1960numerical}), known as CryptoNets. 

Moreover, federated Learning~\cite{mcmahan2017communication} provides a privacy-aware solution for scenarios where data is sensitive (e.g., biomedical records, private images, personal text and speech, and personally identifiable information like location, purchase etc.). 
Federated learning allows multiple clients to train a shared model without collecting their data.
The model training is conducted by aggregating
locally-computed updates and the data in clients
will not be transferred to anywhere for data privacy. In 
vertical setting, two data sets share the same sample ID space but differ in feature space~\cite{hardy2017private,yang2019federated}. For example, consider two different companies in the same city, one is a bank, and the other is an e-commerce company. Their user sets are likely to contain most of the residents of the area, so the intersection of their user space is large. Vertical federated learning on logistic regression, xgboost, multi-tasking learning, neural network, transfer learning, etc, have been previously studied (\cite{acar2018survey,mohassel2017secureml,konevcny2016federated,nock2018entity}). 
However, few research has studied how to train federated GNNs in a privacy-preserving manner when data are vertically partitioned, which popularly exists in practice. To fill in this gap, in this paper, we propose a vertical federated learning framework on graph convolutional network (GCN). In particular, we test our ideas on GraphSage (\cite{chen2020fede,krizhevsky2009learning}). 

Our main contributions are the following:\\

\begin{itemize}
 \item We introduce a vertical federated learning algorithm for graph convolutional network in a privacy-preserving setting to provide solutions for federation problems beyond the scope of existing federated learning approaches;\\

 \item We provide a novel approach which adopts additively homomorphic encryption (HE) to ensure privacy, while maintaining accuracy. Experimental results on three benchmark datasets demonstrate that our algorithm significantly outperforms the GNN models trained on the isolated data and achieves comparable performance with the traditional GNN trained on the combined plaintext data. 
\end{itemize}

\section{Preliminaries}
\subsection{Orthogonal Polynomials}
 The set of functions $\{\Psi_0,\Psi_1,...,\Psi_n\}$ in the interval $[a,b]$ is a set of orthogonal functions with respect to a weight function $w$ if
\begin{equation}
<\Psi_i,\Psi_j>=\int_a^b\Psi_i(x)\Psi_j(x)w(x)dx= \delta_{ij}
\end{equation}

\subsection{Least-Squares Approximation}
{\bf Movivation}: Suppose $f\in C[a, b]$, find a orthogonal polynomial $P_n(x)$ of degree at most $n$ to
approximate $f$ such that $\int_a^b(f(x)-P_n(x))^2dx$ is a minimum (\cite{ali2020poly,carothers1998short}).

Let polynomial $P_n(x)$ be $P_n(x)=\sum_{k=0}^na_kx^k$ which minimizes the error
\begin{equation}
E=E(a_0,a_1,...,a_n)=\int_a^b(f(x)-\sum_{k=0}^n(a_kx^k)^2dx 
\end{equation}
The problem is to find $a_0,...,a_n$ that will minimize $E$. The necessary
condition for $a_0,...,a_n$ to minimize $E$ is $\frac{\partial E}{\partial a_j}=0$, which gives the normal
equations:
\begin{equation}
\sum_{k=0}^na_k\int_a^bx^{j+k}dx=\int_a^bx^jf(x)dx\;\;\; {\rm for}\; j=0,1,...,n.
\end{equation}

We can now find a least-square polynomial approximation of the form
\begin{equation}
    p_n(x)=\sum_{i=0}^{n}b_i\Psi_i(x)
\end{equation}
where $\{\Psi_i\}_{i\in\{0,1...,n\}}$ is a set of orthogonal polynomials. The Legendre polynomials $\{P_0(x)=1,P_1(x)=x,P_2(x)=\frac{3}{2}x^2-\frac{1}{2},...$
is a set of orthogonal functions with respect to the weight function 
$w(x) = 1$ over the interval $[-1, 1]$. In this case, the coefficients of the least-squares polynomial approximation $p_n$ are expressed as follows
\begin{equation}
    b_i=\frac{1}{C_i}\int f(x)P_i(x)dx,\; i\in\{0,1,...,n\}
\end{equation}
Based on this approach, we can approximate the ReLU function using a polynomial of degree two as
\begin{equation}
    p(x)=\frac{4}{3\pi a}x^2+\frac{1}{2}x+\frac{a}{2\pi}
\end{equation}
where $a$ is determined by the data you are working on. 

\subsection{Paillier Homomorphic Encryption}
Homomorphic encryption allows secure computation over encrypted data. To cope with operations in deep neural nets (mainly multiplication and addition), we adopt a well-known partially homomorphic encryption system called {\it Paillier} (\cite{chen2020secure,chi2018privacy}). Paillier homomorphic encryption supports unlimited number of additions between ciphertext, and multiplication between a ciphertext and a scalar constant. Given ciphertext $\{[[M_1]], [[M_2]], \cdots, [[M_g]]\}$ and scalar constants $\{s_1 , s_2 , \cdots ,  s_g \}$, we can calculate
$([[M_1]]\otimes s_1)\oplus ([[M_2]]\otimes s_2)\oplus \cdots \oplus([[M_g]]\otimes s_g)$. without knowing the plaintext message. Here, $[[M_g]]$ represents the ciphertext. $\oplus$ is the homomorphic addition with ciphertext and $\otimes$ is the homomorphic multiplication between a ciphertext and a scalar constant (\cite{barni2006privacy,balle2018improving,chi2018privacy}). 

\subsection{Graph Convolutional Neural Networks}
Graph convolutional network (GCN)~\cite{kipf2017semisupervised} generalizes the operation of convolution from grid data to graph data. The main idea is to generate a node $v$’s representation by aggregating its own features $x_v$ and neighbors’ features $x_u$, where $u\in N(v)$. Different from Recurrent Graph Neural Network, GCN stacks multiple graph convolutional layers to extract high-level node representations. GCN plays a central role in building up many other complex GNN models (\cite{chen2020fede}).
Moreover, many other popular algorithms, such as Graph attention networks (GAT, shown in Figure~\ref{fig:gat})~\cite{gat}, and Graph-SAGE (shown in Figure~\ref{fig:sage})~\cite{NIPS2017_6703}), typically perform well in graph tasks. However, GCN is the transductive method which is not suitable for large graphs.

\begin{figure}
\centering
\includegraphics[width=0.45\textwidth]{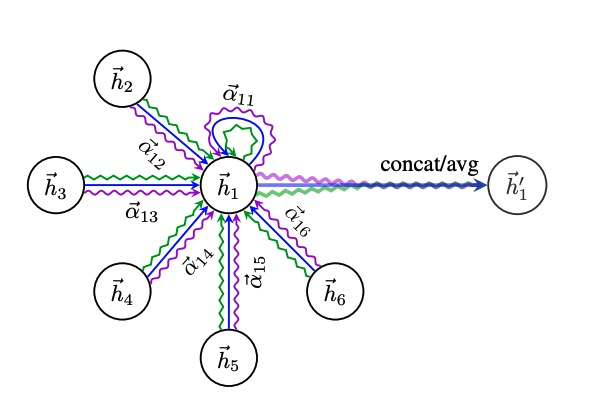}
\caption{A Sample of GAT.}
\label{fig:gat}
\end{figure}

\begin{figure}[t]
\centering
\includegraphics[width=0.45\textwidth]{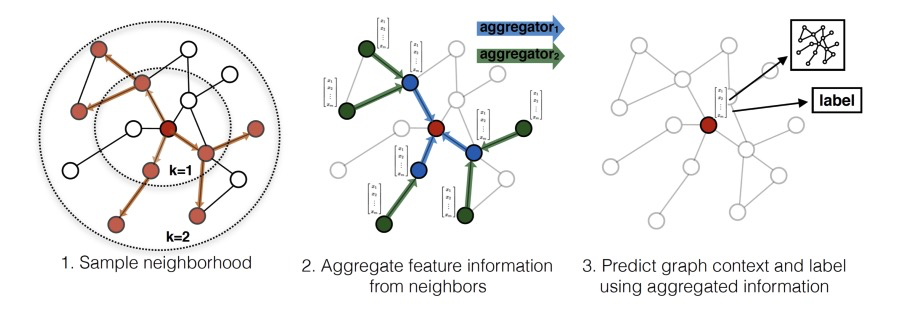}
\caption{A Sample of GraphSAGE.}
\label{fig:sage}
\end{figure}

The main difference between different GNN methods is the aggregation process. 
The general aggregation function can be written as follows:
\begin{equation}
   \hat{h}_{i}^{l}=w_{i,i}^{l}\hat{h}_{i}^{l-1}+w_{i,j}^{l}\hat{h}_{j}^{l-1}~(i \neq j)
\end{equation}

Here, $\hat{h}_{i}^{l}$ is the hidden status in $l$-th layer for the $i$-th node.
Different aggregation functions define different $w_{i,i}$ and $w_{i,j}$. In GraphSAGE method, the aggregation function is mean. And $w_{i,j}$ is calculated by adopting the attention mechanism.

\section{Methodology}
Additively homomorphic encryption and polynomial approximations have been widely used for privacy-preserving machine learning, and the trade-offs between efficiency and privacy by adopting such approximations have been discussed intensively \cite{ali2020poly}. Here we use a second order Taylor approximation for loss and gradients computations:

Define
\begin{equation}
    L(w_1, w_2)=p(w_1h_1+w_2h_2)
\end{equation}
Then
\begin{equation}
    \frac{\partial L}{\partial w_1} = \frac{8h}{3\pi a}(w_1h_1+w_2h_2)*h_1+\frac{1}{2}h_1 
\end{equation}  
\begin{equation}
    \frac{\partial L}{\partial w_2} = \frac{8h}{3\pi a}(w_1h_1+w_2h_2)*h_2+\frac{1}{2}h_2
\end{equation}  

Applying the Homomorphic encryption to $L$ and $\frac{\partial L}{\partial w_i}$ gives
\begin{equation}
[[L]]=\frac{4}{3\pi a}[[(w_1h_1+w_2h_2)^2]]+\frac{1}{2}[[(w_1h_1+w_2h_2)]]+\frac{a}{2\pi}
\end{equation}
\begin{equation}
 [[L_A]]=\frac{4}{3\pi a}[[(w_1h_1)^2]]+\frac{1}{2}[[(w_1h_1]]+\frac{a}{4\pi}
\end{equation}
\begin{equation}
  [[L_B]]=\frac{4}{3\pi a}[[(w_2h_2)^2]]+\frac{1}{2}[[(w_2h_2]]+\frac{a}{4\pi}
\end{equation}
\begin{equation}
 [[L_{AB}]]=\frac{4}{3\pi a}[[2w_1h_1w_2h_2]]
\end{equation}
\begin{equation}
  [[L]] = [[L_A]] + [[L_B]] + [[L_{AB}]]
\end{equation}  
  and 
 \begin{equation}
  [[\frac{\partial L}{\partial w_1}]]= \frac{8h}{3\pi a}[[(w_1h_1+w_2h_2)*h_1]]+\frac{1}{2}[[h_1]]
\end{equation}
\begin{equation}
[[\frac{\partial L}{\partial w_2}]]= \frac{8h}{3\pi a}[[(w_1h_1+w_2h_2)*h_2]]+\frac{1}{2}[[h_2]]
\end{equation}
Here [[x]] is the ciphertext of x. The reason to use quadratic orthogonal polynomials to approximate/replace the relu activation is to preserve the multiplication/addition under homomorphic encryption. 

Suppose that companies A and B would like to jointly train a machine learning model, and their business systems each have their own data. In addition, Company B also has label data that the model needs to predict. We call company A the passive party and the company B the active party. For data privacy and security reasons, the passive party and the active party cannot directly exchange data. In order to ensure the confidentiality of the data during the training process, a third-party collaborator C is involved, which is called the server party. Here we assume the server party is honest-but-curious and does not collude with the passive or the active party, but the passive and active party are honest-but-curious to each other. The passive party trusted the server party is a reasonable assumption since the server party can be played by authorities such as governments or replaced by secure computing node such as Intel Software Guard Extensions (SGX) (\cite{costan2016intel}).

\subsection{Unsupervised loss function}
Existing FL algorithms are mainly for supervised learning. In order to learn useful, predictive representations in a fully unsupervised setting, we adopt the loss function in~\cite{hamilton2017graph}.
That is, the graph-based loss function encourages nearby nodes to have similar representations, while enforcing
that the representations of disparate nodes are highly distinct:
\begin{equation}
    J_{\mathcal{G}} (z_u)=-{\rm log}(\sigma(z_u^Tz_v))-Q\cdot\mathbb{E}_{V_N~P_n(v)}{\rm log}(\sigma(-z_u^Tz_{v_n}))
\end{equation}
where $v$ is a node that co-occurs near $u$ on fixed-length random walk, $\sigma$ is the sigmoid function,
$P_n$ is a negative sampling distribution, and $Q$ defines the number of negative samples. 

\subsection{Algorithms}
In isolated GNNs, both node features and edges are hold by different parties. The first step under vertically data split setting is secure ID alignment, also known as Private Set Intersection (PSI)~\cite{pinkas2014faster}. That is, data holders align their nodes without exposing those that do not overlap with each other. In this work, we assume data holders have aligned their nodes beforehand and are ready for performing privacy preserving GNN training. 

In the case of GCN, the forward propagation and backward propagation of our proposed vertical federated algorithm are presented in Algorithm~\ref{Algorithm:Forward} and Algorithm~\ref{Algorithm:Backward} respectively. Here $N_i$ is the number of edges at party $i$. 
The backward propagation algorithm can be illustrated by Figure~\ref{fig:algo1}.

\begin{algorithm}[t]
 {\bf Active Party}\\
compute $w_1*h_1$ and $N_1$, and send to the server party\\
{\bf Passive Party}\\
compute $w_2*h_2$ and $N_2$, and send to the server party\\
{\bf Server Party}\\
compute $[[w_1*h_1+w_2*h_2]]$ and decrypt them to the active party and the passive party\\
 \caption{FedVGCN: Forward Propagation}\label{Algorithm:Forward}
\end{algorithm}

\begin{algorithm}[t]
 {\bf Initialization:} $w_1, w_2$\\
  {\bf Input:}  learning rate $\eta$, data sample $x$\\
 {\bf Server Party}\\
 create an encryption key pair, send public key to the active party and 
 the passive party.\\
{\bf Passive Party}\\
compute  $[[w_1h_1]], [[L_A]], [[N_1]]$, and send to Active Party, where $N_1$ is the number of neighborhood nodes of the passive party. \\
{\bf Active Party}\\
compute $[[w_2h_2]], [[L]]$,
send $[[w_2h_2]]$ to Passive Party and send $[[L]]$ to the Server Party.\\
{\bf Passive Party}\\
Initialize random noise $\sigma_A$, compute  $[[\frac{\partial L}{\partial w_1}]] +[[\sigma_A]]$ and send to the  Server Party.\\
{\bf Active Party}\\
Initialize random noise $\sigma_B$, compute  $[[\frac{\partial L}{\partial w_2}]] +[[\sigma_B]]$ and send to the Server Party.\\
{\bf Server Party}\\
The Server party decrypts $L$, send $\frac{\partial L}{\partial w_1} +\sigma_A$ to the passive party and $\frac{\partial L}{\partial w_2} +\sigma_B$ to the active party.\\
{\bf return}\\
$w_1,w_2$
 \caption{FedVGCN: Backward Propagation}\label{Algorithm:Backward}
\end{algorithm}

\begin{figure*}
\centering
\includegraphics[width=0.95\textwidth]{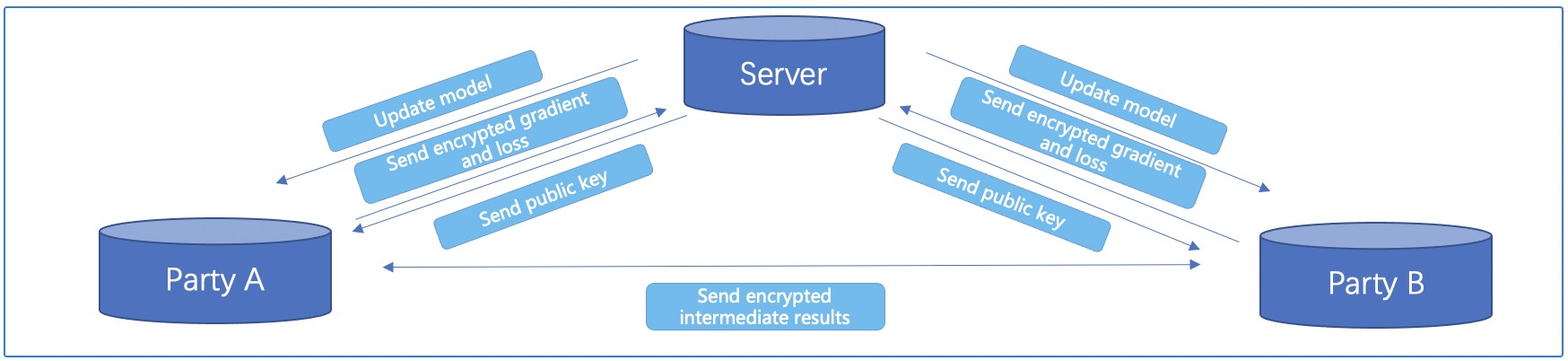}
\caption{FedVGCN: Backward Propagation}
\label{fig:algo1}
\end{figure*}

\section{Security Analysis}

\begin{theorem}
If the number of samples is much greater than the number of features, then the algorithm FedVGCN is secure.
\end{theorem}
\begin{proof}
In the above protocol, the passive party learns its gradient at each step, but this is not enough for the passive party to learn any information from the active party, because the security of scalar product protocol is well-established based on the inability of solving $n$ equations in more than $n$ unknowns. Here we assume the number of samples $N_A$ is much greater than $n_A$, where $n_A$ is the number of features. Similarly, the active party can not learn any information from the passive party. Therefore the security of the protocol is proved. 
\end{proof}

\section{Complexity Analysis}

\subsection{Communication Cost}
The communication cost is analyzed as the total number of messages transmitted between the client and server. We assume a unit message size for encrypted data. For an $n$-layer network, in the forward propagation, the communication cost is $2m(n-1)$, where $m$ is the number of activations in a layer. This is because total $m$ ciphertexts need to be transmitted by the client and the server, for the sum of inputs $z_i$ and activations $a_i$, respectively, at each layer. In the backpropagation, the client needs at most $7 * (m^2 + m + 1)$ messages for model updates between two consecutive layers. Except the final layer, the client interacts with the server to calculate the gradients, which requires transmitting encrypted error $[[\delta_i+1]]c$ in $m$ messages between client and server for each layer. Summing up the costs from the forward and backpropagation, the entire network requires $O(nm^2)$ communication messages for an iteration.

\subsection{Computation Cost}
Arithmetic multiplications and additions are mapped to modular exponentiations and modular multiplications over ciphertext, respectively. Here, we denote such cost of conducting homomorphic arithemtics in Paillier by $p$. For $n$ layers, both forward and backpropagations take $O(nm^2p)$ so the total computation cost is $O(nm^2p)$.

\section{Experiments}
For experiments, we 
mainly focus on GraphSage algorithm for illustration purpose (in this case, we use FedVGraphSage to denote FedVGCN). 
This part aims to answer the following questions:\\

\begin{itemize}
 \item Q1: whether FedVGraphSage outperforms the GraphSage models that are trained on the isolated data.

\item Q2: how does FedVGraphSage behave comparing with the centralized but insecure model that is trained on the plaintext mixed data.
 \end{itemize}
 
\subsection{Experimental Setup}
We first describe the datasets, comparison methods, and 
parameter settings used in our experiments.

\textbf{Datasets}. We use three benchmark datasets which are popularly used to evaluate the performance of GNN, i.e., Cora, Pubmed, and Citeseer~(\cite{balle2018improving}), as shown in Table~\ref{table:1}. We split node features and edges randomly and assume they are hold by two parties. 

\textbf{Comparison methods}.
We compare FedVGraphSage with GraphSAGE models (\cite{chen2020fede}) that are trained using isolated data and mixed plaintext data respectively to answer \textbf{Q1} and \textbf{Q2}. For all the experiments, we split the datasets randomly, use five-fold cross validation and adopt accuracy as the evaluation metric. 

\textbf{Parameter settings}.
For the models which are trained on isolated data, we use relu as the activation function. For the deep neural network on server, we set the dropout rate to 0.5 and network structure as $(64, 64, |C|)$,
where $|C|$ is the number of classes. We set the learning rate as $10^{-5}$.

\begin{table}[h!]
\centering
\begin{tabular}{||c c c c c||} 
 \hline
 Dataset & \#Node & \#Edge & \#Feature & \#Classes \\ [0.5ex] 
 \hline\hline
 Cora & 2708 & 5409 & 1433 & 7\\ 
 \hline
 Pubmed & 19717 & 44338 & 500 & 3 \\
 \hline
 Citeseer & 3327 & 4732 & 3703 & 6\\ [1ex] 
 \hline
\end{tabular}
\caption{Dataset statistics}
\label{table:1}
\end{table}

\begin{table}[h!]
\centering
\begin{tabular}{||c c c c||} 
 \hline
 Dataset & Cora &  Pubmed &  Citeseer \\ [0.5ex] 
 \hline\hline
 $GraphSage_A$ &  0.5222  & 0.6936 & 0.4630 \\ 
 \hline
  $GraphSage_B$ & 0.4867 & 0.6801  & 0.5510  \\
   \hline
FedVGraphSage &  0.6770 & 0.7830  & 0.6820 \\
 \hline
 $GraphSage_{A+B}$ & 0.7080 & 0.7890  & 0.6983 \\ [1ex] 
 \hline
\end{tabular}
\caption{Accuracy comparison on three datasets}
\label{table:2}
\end{table}

\subsection{Comparison Results and Analysis}

{\bf Answer to Q1}:  We compare FedVGraphSage with the GraphSages that are trained on the isolated feature and edge data, i.e.,  $GraphSage_{A}$ and $ GraphSage_{B}$. From Table~\ref{table:2}, we can find that, FedVGraphSage significantly outperforms $GraphSage_{A}$ and $ GraphSage_{B}$ on all the three datasets. Take Cora for example, our FedVGraphSage improves $GraphSage_{A}$ and $ GraphSage_{B}$ by as high as $29.6\%$ and $39.1\%$, in terms of accuracy.

{\bf Analysis}: $GraphSage_{A}$ and $ GraphSage_{B}$ can only use partial feature and edge information hold by $A$ and $B$, respectively. In contrast, FedVGraphSage provides a solution for $A$ and $B$ to train GraghSages collaboratively without compromising their own data. By doing this, FedVGraphSage can use the information from the data of both $A$ and $B$ simultaneously, and therefore achieve better performance.

{\bf Answer to Q2}: We then compare FedVGraphSage with $GraphSage_{A+B}$
that is trained on the mixed plaintext data. It can be seen from Table~\ref{table:2} that FedVGraphSage has comparable performance with $GraphSage_{A+B}$, e.g., 0.6770 vs. 0.7080 on Cora dataset, 0.7830 vs. 0.7890 on Pubmed dataset, 0.6820 vs. 0.6983 on Citeseer dataset.

{\bf Analysis}: First, we use Paillier Homomorphic Encryption for $A$ and $B$ to securely communicate intermidiate results, as described in Algorithm~\ref{Algorithm:Forward} and Algorithm~\ref{Algorithm:Backward}. Second, we use the quadratic orthogonal polynomials activation functions to preserve the sum and multiplication operations which acts on the encrypted results. Therefore, FedVGraphSage has comparable performance with $GraphSage_{A+B}$.

\section{Conclusion}
In this work, we introduce a novel vertical federated learning algorithm for graph neural network. We adopt additively homomorphic encryption to ensure privacy, while maintaining accuracy. Experimental results on three benchmark datasets demonstrate that our algorithm significantly outperforms the GNN models trained on the isolated data and has comparable performance with the traditional GNN trained on the combined plaintext data. In this paper, we mainly investigate the vertical federated GNN by adopting GraphSAGE, and in the future, we will extend our method to more complex GNN models, such as GAT, GIN, etc. We also plan to apply our method to real industrial applications.

\bibliography{references.bib}
\bibliographystyle{named}

\end{document}